%% file: iclr2027_conference.tex
\documentclass{article}

\IfFileExists{iclr2027_conference.sty}{
  \usepackage{iclr2027_conference,times}
}{
  \usepackage[
    letterpaper,
    textwidth=5.5in,
    textheight=9in,
    top=1in,
    left=1.5in
  ]{geometry}
  \usepackage{times}
  \usepackage[round,authoryear]{natbib}
}

\usepackage{amsmath,amssymb,amsthm,mathtools}
\usepackage{booktabs}
\usepackage{multirow}
\usepackage{microtype}
\usepackage{graphicx}
\usepackage{xcolor}
\usepackage[table]{xcolor}
\usepackage{url}
\usepackage{hyperref}
\hypersetup{hidelinks}
\usepackage{enumitem}

\usepackage{wrapfig}
\definecolor{ecrpositive}{HTML}{2F6DB0}
\definecolor{ecrnegative}{HTML}{D44E41}

\definecolor{improvegreen}{RGB}{28,128,76}
\definecolor{oursgray}{RGB}{225,225,225}

\newcommand{\ours}{UECR-GRPO}
\newcommand{\T}{\mathrm{T}}
\newcommand{\old}{\mathrm{old}}
\newcommand{\task}{\mathrm{task}}
\newcommand{\KL}{\mathrm{KL}}
\newcommand{\E}{\mathbb{E}}

\newcommand{\clip}{\operatorname{clip}}
\newcommand{\softmax}{\operatorname{softmax}}

\newtheorem{proposition}{Proposition}
\newtheorem{lemma}{Lemma}
\newtheorem{corollary}{Corollary}

\title{
  When and Where to Trust the Teacher:
  Unifying On-Policy Distillation and GRPO through
  Entropy-Calibrated Credit Assignment
}

\author{
Jie Zhang$^{1}$,\quad
Jingxiao Yang$^{2}$,\quad
Zhehao Huang$^{1}$,\quad
Yuhang Liu$^{1}$,\quad
Xiaolin Huang$^{1,\dagger}$
\\[-1pt]
\normalfont
$^{1}$Shanghai Jiao Tong University
\\
$^{2}$Zhejiang University
}
\date{}

\iclrfinalcopy

\begin{document}

\maketitle

\begingroup
\renewcommand{\thefootnote}{\fnsymbol{footnote}}
\footnotetext[2]{Corresponding author.}
\endgroup

\lhead{Preprint}

\input{content/abstract}

\input{content/introduction}
\input{content/related_work}
\input{content/preliminaries_and_method}
\input{content/experiments}
\input{content/conclusion}

\clearpage
\bibliography{iclr2027_conference}
\IfFileExists{iclr2027_conference.bst}{
  \bibliographystyle{iclr2027_conference}
}{
  \bibliographystyle{plainnat}
}

\clearpage
\appendix
\input{content/appendix}

\end{document}

%% file: content/abstract.tex
\begin{abstract}
Reinforcement learning with verifiable rewards (RLVR) supervises mathematical
reasoning through final-answer correctness, but provides little guidance on
individual tokens. On-policy distillation (OPD) supplies dense feedback on
student-generated responses, yet teacher preference need not reflect
correctness. Recent hybrids combine OPD and verifier-derived advantages or
reweight task credit using teacher ratios. However, teacher guidance enters
after verifier-based group normalization, and token reweighting need not
preserve the total task credit assigned to each response.
We introduce Unified Entropy-Calibrated Credit Redistribution for GRPO
(UECR-GRPO), which integrates verifier and teacher signals within a single
GRPO-style update at both the response and token levels. \emph{Path-Utility Unification} (PUU) combines verifier reward and a
teacher-to-anchor path log-ratio in a single KL-regularized objective.
Its on-policy implementation uses a length-normalized teacher score and
combines both rewards before group normalization and PPO clipping, allowing
teacher evidence to influence the response ranking.
\emph{Entropy-Calibrated Redistribution} (ECR) then uses the signed
teacher--old-policy token gap to redistribute the verifier-derived component.
Full-vocabulary teacher entropy attenuates uncertain guidance, while a
response-wise zero-sum projection preserves the total task credit and its
token-wise sign before clipping.
Across five mathematical reasoning benchmarks, UECR-GRPO achieves
average \(\mathrm{Avg@12}\) accuracies of 17.21\% and 65.09\% with
Qwen3-1.7B and Qwen3-4B students, respectively, exceeding the strongest
baseline at each scale by 0.89 and 0.56 percentage points.
\end{abstract}

%% file: content/introduction.tex
\section{Introduction}

Reinforcement learning with verifiable rewards has become a practical route
for eliciting mathematical reasoning from language models
\citep{shao2024deepseekmath,guo2025deepseekr1,yu2025dapo}.  Group Relative
Policy Optimization (GRPO) is particularly attractive because it replaces a
learned value model with comparisons among multiple responses to the same
problem \citep{shao2024deepseekmath}.  The verifier checks the final answer
against an external correctness criterion, but provides coarse credit.
One response-level advantage is normally broadcast to every generated token,
so a decisive algebraic step, a harmless stylistic token, and the first local
error all receive the same advantage.  Moreover, when responses within a group receive identical binary rewards, the group-relative task signal vanishes.

On-policy distillation (OPD) provides complementary information
\citep{agarwal2024gkd,yang2026gopd}.  A stronger teacher scores the student's
own prefixes and therefore supplies dense feedback exactly on the states the
student visits.  Such feedback can distinguish trajectories even when their
verifier rewards are tied.  Yet the teacher may prefer an ultimately
incorrect path, discourage an unfamiliar but valid derivation, or transfer
style rather than mathematical substance.  The verifier therefore indicates
\emph{whether} a response succeeds but not \emph{where}; the teacher provides
local preference but does not define terminal correctness.

Figure~\ref{fig:signal-complementarity} shows that teacher evidence adds
resolution without reliably defining response quality. Across checkpoints, teacher scores distinguish 84.3--98.4\% of verifier-degenerate groups under the fixed threshold. However, teacher-induced ordering conflicts with verifier correctness on 39.1--50.8\% of correct--incorrect response pairs. Among 347 judgeable verifier-tied pairs, teacher preference agrees with blinded Opus 4.8 process-quality judgments in only 56.8\% of cases. Teacher
evidence is thus informative as a policy-relative signal, but is neither a
replacement for terminal verification nor a reliable process-quality label by
itself. This motivates retaining the verifier as the task-defining signal
while using the teacher only to supply additional resolution.

\begin{figure*}[t]
\centering
\includegraphics[width=0.97\textwidth]
{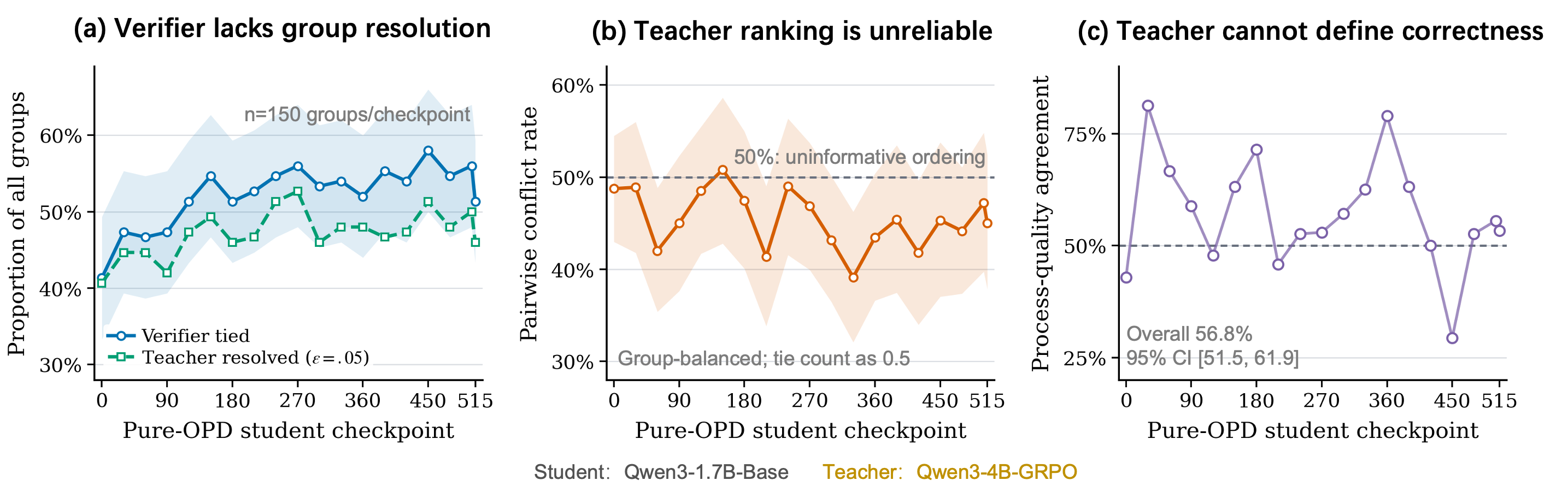}
\caption{Signal complementarity across 19 Pure-OPD checkpoints. At each
checkpoint, the student generates eight responses for each of 150 prompts.
(a) reports verifier-degenerate groups and the subset whose teacher-score
range exceeds \(\epsilon_{\mathrm{gap}}=0.05\). (b) reports the
group-balanced conflict rate over correct--incorrect pairs, with
teacher-score ties counted as \(0.5\). (c) reports agreement between teacher
preference and a blinded independent process-quality judge on verifier-tied
pairs.}
\label{fig:signal-complementarity}
\end{figure*}

Recent methods combine GRPO and teacher supervision beyond a naive sum of
independent losses. Distilled RL multiplicatively reweights each positive,
normalized GRPO advantage using teacher-to-student token ratios
\citep{wang2026distilledrl}. Consequently, when identical verifier rewards
produce a zero GRPO advantage, the reweighted update remains zero. ATOD instead
adds an annealed, turn-weighted token-level OPD advantage to an already
normalized GRPO advantage \citep{tan2026atod}, so its teacher signal remains
active when verifier rewards are tied. However, teacher evidence enters after
group normalization as an additive token-level objective rather than as part
of a joint response utility. Moreover, its token weighting is not constrained
to preserve the response-average update magnitude. This leaves open how to
incorporate teacher evidence into group-relative comparison and localize task
credit without altering response-level scale.

To address these questions, we propose \emph{Unified Entropy-Calibrated Credit
Redistribution for GRPO} (UECR-GRPO). Its trajectory-level component,
\emph{Path-Utility Unification} (PUU), incorporates teacher evidence into
group-relative comparison. Viewing each response as an autoregressive path,
the sum of teacher-to-anchor token log-ratios gives an exact path log-density
ratio. We treat this quantity as an implicit teacher reward and combine it
with verifier utility in a single KL-regularized objective. The resulting
Gibbs optimum defines a target distribution shaped jointly by task reward,
teacher preference, and an anchor policy. In the on-policy implementation,
verifier reward and a length-normalized teacher score are combined before
group normalization, allowing teacher evidence to influence the sampled
response ranking before a clipped GRPO-style update. Verifier utility remains
explicit, and an offline sensitivity audit shows that our chosen setting lies
well within the empirically observed safe region for correct--incorrect
ordering reported in Appendix~\ref{app:alpha-sweep}.

However, trajectory unification alone does not resolve task-credit
localization without uncontrolled response-level scaling, because PUU still
assigns one scalar advantage to each response. We therefore introduce
\emph{Entropy-Calibrated Redistribution} (ECR). The signed log-probability gap
between the teacher and the old policy provides a direction for adjusting
verifier-derived credit, while the teacher's full-vocabulary entropy measures
predictive uncertainty and attenuates the guidance under diffuse next-token
distributions. A response-wise zero-sum projection redistributes the task
component while preserving its total credit and token-wise sign before
clipping. This allows teacher-preferred tokens in failed responses to receive
less negative task credit without becoming positive imitation targets,
separating local credit assignment from response-level scaling.

Our contributions are:
\begin{itemize}[leftmargin=*,topsep=2pt,itemsep=1pt]
  \item We formulate PUU through a joint KL-regularized objective that
  combines verifier utility and teacher-to-anchor improvement, and derive
  its Gibbs-optimal path distribution. Its on-policy implementation
  combines both signals before group normalization, allowing teacher
  evidence to influence response ranking.

  \item We introduce ECR to redistribute verifier-derived credit using
an entropy-attenuated teacher preference signal. A zero-sum projection
preserves each response's total task credit and token-wise sign before
clipping, separating token credit assignment from response-level scaling.

  \item Experiments on mathematical reasoning benchmarks with 1.7B and
  4B students show average accuracy gains of 0.89 and 0.56 percentage
  points over the strongest baseline, respectively.
\end{itemize}

%% file: content/related_work.tex
\section{Related Work}

\paragraph{RLVR and group-relative optimization.}
PPO uses a clipped policy-gradient surrogate \citep{schulman2017ppo}, while
GRPO derives advantages from within-prompt reward statistics
\citep{shao2024deepseekmath}, whose normalization affects the weighting of
successful and failed samples \citep{mroueh2025grpo}. DAPO improves clipping,
sampling, token-level loss reduction, and length handling
\citep{yu2025dapo}, and GSPO introduces sequence-level importance ratios and
clipping \citep{zheng2025gspo}. These methods refine policy optimization, but
token credit remains determined by sequence-level rewards.

\paragraph{On-policy distillation and joint RL--KD.}
Generalized Knowledge Distillation trains on student-generated sequences to
reduce train--inference mismatch \citep{agarwal2024gkd}. KDRL combines
reverse-KL distillation with rule-based rewards \citep{xu2025kdrl}, while
G-OPD formulates OPD as dense KL-regularized RL and decouples implicit-reward
strength from regularization \citep{yang2026gopd}. Reward-gated and sign-gated
variants use verifier information to decide when teacher supervision applies
\citep{akhondzadeh2026rgopd,xu2026sgopd}. PUU instead forms a shared response
utility before computing the group-relative advantage.

\paragraph{Teacher-guided credit assignment.}
Distilled RL uses clipped, geometrically normalized teacher-to-old-policy
ratios to reweight positive GRPO advantages, reverting to ordinary RL
otherwise \citep{wang2026distilledrl}. ECR instead handles both advantage
signs, calibrates guidance by teacher entropy, and preserves the arithmetic
sum of verifier-derived credit rather than the product of token weights.

ATOD combines annealed token-level OPD and GRPO in one clipped update, with
T-DUR reweighting turns using teacher--student disagreement and student
uncertainty \citep{tan2026atod}. Unlike PUU, its teacher evidence enters only
after group normalization and therefore does not affect the response rankings
or group statistics used to construct the GRPO advantage.

RLSD scales updates by teacher--student differences \citep{yang2026rlsd},
RLCSD contrasts correct- and wrong-hint teachers \citep{pan2026rlcsd}, and
StepOPSD redistributes supervision over action segments
\citep{zhang2026stepopsd}. TASPO uses mean-preserving privileged-information
weights at the action level
\citep{yang2026reconcilingprocesssupervisionoutcomebased}, SGCD constructs
sibling-based stepwise references \citep{ding2026sgcd}, and SC-GRPO
multiplicatively weights GRPO with self-conditioned KL
\citep{shan2026scgrpo}. ECR instead uses entropy-calibrated teacher gaps for
token-level redistribution while preserving each response's verifier-credit
total.

\paragraph{Teacher uncertainty.}
Entropy-Aware OPD addresses unstable reverse KL in high-entropy teacher
regions using forward KL \citep{jin2026eopd}. ATOD's T-DUR uses student
sampled-token surprisal to strengthen OPD supervision \citep{tan2026atod}.
ECR gates local adjustments with full-vocabulary teacher entropy and recovers
the broadcast PUU advantage as confidence vanishes.

\paragraph{KL-regularized control and inference.}
KL control represents a policy as a controlled change of measure from anchor
dynamics \citep{todorov2006lmdp}. Control-as-inference yields reward-tilted
trajectory distributions and connects optimal control with variational
inference \citep{levine2018inference,theodorou2010pathintegral}. PUU applies
this formulation to joint verifier and teacher utility, replacing the exact
path log-ratio sum with a token mean before group-relative normalization.

%% file: content/preliminaries_and_method.tex
\section{Preliminaries and Problem Formulation}
\label{sec:prelim}

\paragraph{Setting and notation.}
Given a prompt \(x\), a student policy generates a response
\(y=(y_1,\ldots,y_L)\).  The behavior policy \(\pi_\old\) collects the current
batch, \(\pi_\theta\) is the student being updated, and \(\pi_\T\) is a frozen
teacher.  A separate frozen \(\pi_{\mathrm{ref}}\), when enabled, regularizes
the actor update.  For response \(i\) and token \(t\),
\(s_{i,t}=(x,y_{i,<t})\), and \(m_{i,t}\) selects actor-controlled response
tokens while excluding prompts and padding.  We use \(P_0\) for the abstract
anchor path measure in the theoretical analysis.  In the on-policy algorithm
we instantiate \(P_0=P_\old\).

\subsection{GRPO: task-valid but sequence-level supervision}

For each prompt, \(\pi_\old\) samples \(G\) responses.  Given verifier reward
\(R_i^\task\), GRPO forms
\begin{equation}
A_i^{\mathrm{GRPO}}
=\frac{R_i^\task-\operatorname{Mean}_{j\le G}R_j^\task}
{\operatorname{Std}_{j\le G}(R_j^\task)+\epsilon},
\qquad y_i\sim\pi_{\old,T_r}(\cdot\mid x).
\label{eq:grpo-adv}
\end{equation}
Here, \(T_r\) denotes the rollout temperature. GRPO avoids a value model but
broadcasts one advantage across each response, preventing local credit
assignment. When group reward variance is zero, its verifier-derived update
vanishes entirely.

\subsection{OPD: dense preference on student-visited states}

At state \(s_{i,t}\), define the raw-temperature teacher-to-old-policy gap
\begin{equation}
\delta_{i,t}
=\log\pi_{\T,T=1}(y_{i,t}\mid s_{i,t})
-\log\pi_{\old,T=1}(y_{i,t}\mid s_{i,t}).
\label{eq:teacher-token-reward}
\end{equation}

The gap measures the teacher's preference for each sampled token relative
to the behavior policy. Here, \(\pi_\old\) denotes the policy that generated
the responses and remains fixed during the batch update, while
\(\pi_\theta\) is optimized. Dense preference is not process correctness.  A positive gap only says that
the teacher assigns more probability to the realized token than the old
student does.  It does not establish that the token is mathematically valid,
causally important, or part of a successful final response.

\subsection{Three ways of combining reward and teacher evidence}

The relevant methods differ mainly in \emph{when} teacher information enters
the update.

\begin{enumerate}[leftmargin=*,topsep=2pt,itemsep=1pt,parsep=0pt,partopsep=0pt]
\item \textbf{Independent objectives.}
A naive hybrid adds a clipped GRPO surrogate and an independently constructed
OPD surrogate. Because clipping is nonlinear, adding the two clipped losses
is not equivalent to combining their evidence first. Teacher evidence cannot
alter the reward ranking used to form the GRPO advantage.

\item \textbf{Post-normalization integration.}
ATOD combines an OPD token advantage with an already normalized GRPO advantage
before one clipped actor update \citep{tan2026atod}, while Distilled RL
multiplies the normalized GRPO advantage by a normalized teacher ratio
\citep{wang2026distilledrl}. These are integrated actor updates rather than
independent-loss baselines, but the verifier ranking is fixed before teacher
information enters.

\item \textbf{Pre-normalization trajectory unification.}
We instead combine verifier reward and teacher evidence into a single response
utility before computing group-relative advantages. This response-level
interpretation motivates the path-space formulation below. Token credit is
then handled separately by redistributing the verifier-derived component
within each response.
\end{enumerate}

\section{Unified Trajectory Utility and Constrained Token Credit}
\label{sec:method}

\paragraph{Overview.}
The method follows an ordered two-layer design.  PUU first combines verifier
reward and a response-level teacher improvement before group normalization,
yielding one trajectory advantage.  ECR then modifies only the placement of
the verifier-derived component, using teacher uncertainty and a zero-sum
constraint.  Separating these levels lets the experiments ask two distinct
questions: whether pre-normalization trajectory unification helps, and whether
constrained localization adds value beyond that unified objective.

\begin{figure*}
\centering
  \includegraphics[width=0.97\textwidth]{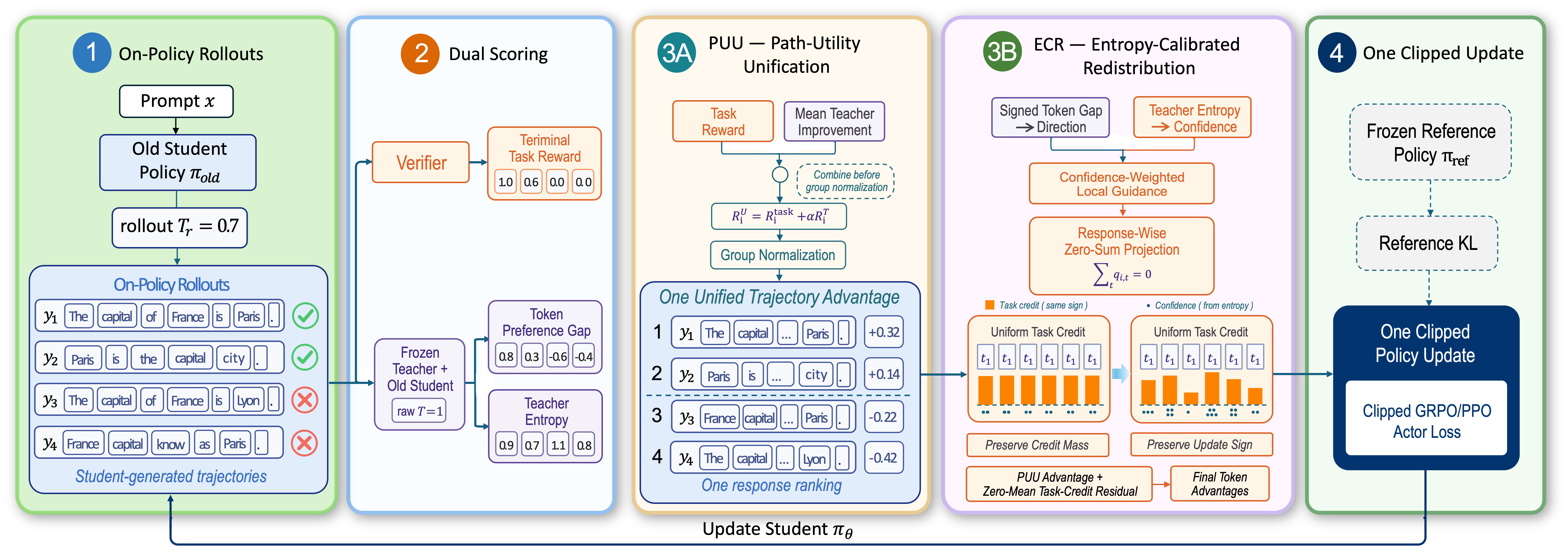}
\caption{Method overview. PUU combines verifier reward and teacher
improvement before group normalization. ECR then redistributes only the
verifier-derived component using a signed teacher gap, teacher entropy, and a
response-wise zero-sum projection.}
\label{fig:overview}
\end{figure*}

\subsection{PUU: Path-Utility Unification}
\label{sec:puu}

Let \(P_0\) and \(P_T\) be the anchor and teacher path distributions for a
fixed prompt.  Autoregressive factorization gives
\begin{equation}
\sum_{t=1}^{L}\log
\frac{\pi_T(y_t\mid x,y_{<t})}{\pi_0(y_t\mid x,y_{<t})}
=\log\frac{P_T(y\mid x)}{P_0(y\mid x)}.
\label{eq:path-ratio-main}
\end{equation}
Thus the summed OPD signal is an exact path log-density ratio rather than an
unrelated auxiliary loss.

For a candidate path distribution \(Q\), define
\begin{equation}
J(Q)=
\E_{y\sim Q}\!\left[
R_\task(x,y)+\alpha\log\frac{P_T(y\mid x)}{P_0(y\mid x)}
\right]
-\beta\KL\!\left(Q(\cdot\mid x)\,\|\,P_0(\cdot\mid x)\right).
\label{eq:path-objective}
\end{equation}
The verifier specifies task utility, the log-ratio measures teacher
improvement over the anchor, and the KL term controls departure from that
anchor.  The unique maximizer is
\begin{equation}
Q^*(y\mid x)\propto
\exp\!\left(\frac{R_\task(x,y)}{\beta}\right)
P_T(y\mid x)^{\alpha/\beta}
P_0(y\mid x)^{1-\alpha/\beta}.
\label{eq:gibbs}
\end{equation}
Setting \(\alpha=0\) recovers reward-regularized RL; setting
\(R_\task=0\) and \(\alpha=\beta\) recovers \(P_T\); and
\(R_\task=0,\alpha>\beta\) yields reward extrapolation beyond the teacher in
the teacher-over-anchor direction.  These are properties of the
distributional target.  They do not imply that the finite-sample clipped
algorithm is identical to each corresponding training procedure.

\subsubsection{On-policy group-relative realization}

Direct normalization over all language-model paths is intractable. At each
iteration, we set \(P_0=P_\old\), sample \(G\) responses from \(\pi_\old\),
and score their realized tokens with \(\pi_T\) and \(\pi_\old\). The teacher
term thus measures improvement over the behavior policy on the sampled
responses. The exact path reward in Eq.~\eqref{eq:path-ratio-main} is a token
sum whose magnitude grows mechanically with response length. We therefore use
the length-normalized surrogate
\begin{equation}
R_i^T=\frac{\sum_t m_{i,t}\delta_{i,t}}{\sum_t m_{i,t}},
\qquad
R_i^U=R_i^\task+\alpha R_i^T.
\label{eq:unified-reward}
\end{equation}
For variable-length responses, \(R_i^T\) is no longer an exact path
log-density ratio. It is a practical surrogate that prevents response length
from implicitly changing the scale of the teacher contribution.

PUU forms the joint utility before computing the group-relative advantage
\begin{equation}
\mu_U=\frac1G\sum_{j=1}^{G}R_j^U,\qquad
\sigma_U=\operatorname{Std}_{j=1}^{G}(R_j^U),\qquad
A_i^U=\frac{R_i^U-\mu_U}{\sigma_U+\epsilon}.
\label{eq:unified-advantage}
\end{equation}
Teacher evidence can therefore change the ordering among sampled responses,
while the combined utilities also determine the group mean and scale used by
the clipped update. In post-normalization hybrids, the verifier-derived
ordering has already been fixed when teacher evidence is introduced.

To prepare for token-level redistribution, we decompose this single unified
advantage into task and teacher components that share the same scale
\begin{equation}
A_i^{\task\mid U}=\frac{R_i^\task-\mu_\task}{\sigma_U+\epsilon},
\qquad
A_i^{T\mid U}=\frac{R_i^T-\mu_T}{\sigma_U+\epsilon}.
\label{eq:unified-components}
\end{equation}
Since centering is linear, substituting Eq.~\eqref{eq:unified-reward} into
Eq.~\eqref{eq:unified-advantage} gives
\(A_i^U=A_i^{\task\mid U}+\alpha A_i^{T\mid U}\) exactly. These are additive
components of one advantage rather than two independently normalized
advantages. Their shared denominator preserves the intended coefficient
\(\alpha\), whereas separate standardization would make its effective value
vary across groups. For verifier-degenerate groups, the unified advantage
approximately reduces to a group-standardized response-level OPD score. This
decomposition allows ECR to redistribute verifier-derived credit while
leaving the teacher component unchanged.

\subsection{ECR: Entropy-Calibrated Redistribution}
\label{sec:ecr}

Figure~\ref{fig:token-credit-diagnosis} examines token-level teacher
signals at initialization. We sample eight responses for each of 300 training prompts from the 1.7B student, yielding 2,400 responses, and evaluate every realized token under the frozen 4B teacher and the behavior policy. When tokens are grouped by teacher entropy, high-entropy prefixes assign lower probability to the realized token and contain most large teacher--old-policy log-probability gaps. Without response-wise centering, using these local signals as token weights also moves their arithmetic mean away from one. This changes the total verifier-derived credit instead of only redistributing it across tokens. The results motivate separate mechanisms for calibrating teacher uncertainty and preserving task credit.

\begin{figure*}[t]
\centering
\includegraphics[width=0.97\textwidth]
{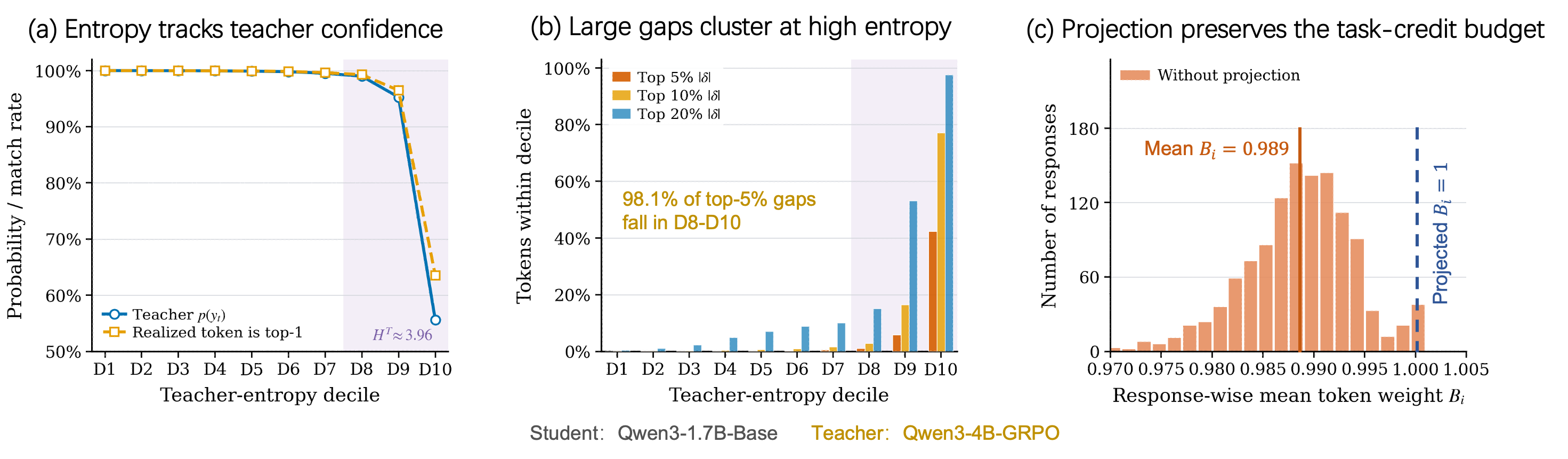}
\caption{Initialization-time token-credit diagnosis. (a) relates teacher
entropy to the probability and top-1 rank of the realized token. (b) shows
where tokens with large absolute teacher--old-policy gaps occur across
entropy deciles. (c) reports the response-wise mean token weight before
projection, while zero-sum projection fixes this mean at one.}
\label{fig:token-credit-diagnosis}
\end{figure*}

\subsubsection{Direction and confidence}

PUU ranks responses but still broadcasts one scalar to all their tokens. Let
\(s_i=\operatorname{sign}(A_i^{\task\mid U})\). We construct a bounded local
direction from the teacher--old-policy gap and calibrate it using the
teacher's uncertainty
\begin{equation}
d_{i,t}=\tanh\!\left(\frac{s_i\delta_{i,t}}{2\tau_\delta}\right),
\qquad
c_{i,t}=\exp\!\left(-\frac{H^T_{i,t}}{\tau_H}\right).
\label{eq:local-signal}
\end{equation}
Here, \(H^T_{i,t}\) is the frozen teacher's full-vocabulary entropy at the
current prefix. The positive scale \(\tau_\delta\) controls how quickly the
direction saturates as the gap grows, while \(\tau_H\) controls the strength
of entropy attenuation. A larger gap increases \(|d_{i,t}|\), whereas higher
teacher entropy decreases \(c_{i,t}\).

\begin{wraptable}{r}{0.34\columnwidth}
\vspace{-10pt}
\centering
\caption{ECR credit directions.}
\label{tab:ecr-four-cases}
\vspace{2pt}
\scriptsize
\setlength{\tabcolsep}{2.8pt}
\renewcommand{\arraystretch}{1.08}
\begin{tabular}{@{}c@{\hspace{5pt}}cc@{}}
\toprule
& \(\delta_{i,t}>0\) & \(\delta_{i,t}<0\) \\
\midrule
\(A_i^{\task\mid U}>0\)
& \textcolor{ecrpositive}{\(\uparrow\)}
& \textcolor{ecrpositive}{\(\downarrow\)} \\[4pt]
\(A_i^{\task\mid U}<0\)
& \textcolor{ecrnegative}{\(\uparrow\)}
& \textcolor{ecrnegative}{\(\downarrow\)} \\
\bottomrule
\end{tabular}
\vspace{-6pt}
\end{wraptable}

For \(A_i^{\task\mid U}>0\), teacher-preferred tokens have positive local
direction. When \(A_i^{\task\mid U}<0\), \(s_i\) reverses this direction, so
teacher-preferred tokens can receive less negative task credit without
becoming unconditional imitation targets. Entropy affects only the strength
of this proposal and does not determine its direction or correctness.

\subsubsection{Zero-sum response-wise projection}

Applying the local signal directly as a token weight would change both the
allocation and the total amount of task credit as shown in Figure~\ref{fig:token-credit-diagnosis}. We instead subtract its confidence-weighted response mean
\begin{equation}
\mu_i^c=
\frac{\sum_t m_{i,t}c_{i,t}d_{i,t}}
{\sum_t m_{i,t}c_{i,t}},\qquad
q_{i,t}=\tfrac12m_{i,t}c_{i,t}(d_{i,t}-\mu_i^c),
\qquad
w_{i,t}=1+\rho q_{i,t},\quad 0\le\rho<1.
\label{eq:projection-weights}
\end{equation}
Thus, \(q_{i,t}\) measures each token's direction relative to the
confidence-weighted response baseline. Tokens above this baseline receive
larger task weights and those below it receive smaller weights, while \(\rho\)
controls the overall redistribution strength.

For every non-empty response,
\(\sum_{t:m_{i,t}=1}q_{i,t}=0\), so the arithmetic mean of \(w_{i,t}\) over
valid tokens remains exactly one. Moreover, \(|q_{i,t}|\le1\) and
\(\rho<1\) keep all weights positive. ECR therefore preserves both the
response-wise additive mass and token-wise sign of the task component,
although the full PUU advantage may still have a different sign.


\subsubsection{Final actor advantage}

The final advantage adds the zero-mean redistribution residual to the PUU
advantage
\begin{equation}
A_{i,t}^{\mathrm{final}}
=A_i^U+\rho A_i^{\task\mid U}q_{i,t}
=A_i^{\task\mid U}(1+\rho q_{i,t})
+\alpha A_i^{T\mid U}.
\label{eq:final-advantage}
\end{equation}
This form makes clear that ECR redistributes only the verifier-derived
component, while the teacher component introduced by PUU remains unchanged.
The actor minimizes
\begin{equation}
\mathcal L
=\mathcal L_{\mathrm{PG}}(A^{\mathrm{final}})
+\beta_{\mathrm{ref}}\mathcal L_{\mathrm{low\mbox{-}var\mbox{-}KL}}.
\label{eq:final-objective}
\end{equation}
The reference KL enters only through the actor objective and is not included
in either response reward. When \(\rho=0\), the method reduces exactly to
PUU. The same reduction occurs when confidence is negligible or all local
directions are identical, since the projected residual then vanishes.

These properties address the two observations in
Figure~\ref{fig:token-credit-diagnosis}. Entropy calibration limits the
influence of large teacher gaps at uncertain prefixes, while the zero-sum
projection redistributes token credit without changing the response-wise
task-credit budget.

\begin{table}[t]
\centering
\scriptsize
\setlength{\tabcolsep}{3.2pt}
\setlength{\belowcaptionskip}{5pt}
\renewcommand{\arraystretch}{1.08}
\caption{Main mathematical reasoning results for Qwen3 students. Block
headers specify the student and frozen teacher. All values are
\(\mathrm{Avg@12}\) accuracy (\%) under a fixed 12-sample decoding setup.
The average weights all five benchmarks equally. A dash denotes an
unevaluated setting.}
\label{tab:main}

\resizebox{\linewidth}{!}{%
\begin{tabular}{@{}lcccccc@{}}
\toprule
\multirow{3}{*}{Method}
& \multicolumn{6}{c}{\(\mathrm{Avg@12}\) accuracy (\%)}\\
\cmidrule(l){2-7}
& AIME24
& AIME25
& AMC23
& \shortstack{HMMT25\\Feb}
& \shortstack{HMMT25\\Nov}
& Avg.\\
\midrule

\multicolumn{7}{@{}l}{%
\textbf{Qwen3-1.7B-Base student} /
\textit{Qwen3-4B-GRPO teacher}}\\
\addlinespace[2pt]

Initial student \citep{yang2025qwen3}
& 1.53
& 1.75
& 12.02
& 0.00
& 1.94
& 3.45\\

Vanilla-GRPO \citep{shao2024deepseekmath,mroueh2025grpo}
& 7.78
& 6.39
& 37.50
& 0.28
& \textbf{5.28}
& 11.45\\

Vanilla-PG-OPD \citep{agarwal2024gkd}
& 9.17
& 7.22
& 39.69
& 0.28
& \underline{4.72}
& 12.22\\

Naive-GRPO+PG-OPD
& 7.50
& 6.94
& 38.54
& 0.28
& 3.89
& 11.43\\

Distilled RL \citep{wang2026distilledrl}
& 14.44
& \underline{9.31}
& 48.23
& \underline{4.44}
& 4.17
& 16.12\\

ATOD-aligned \citep{tan2026atod}
& \underline{14.72}
& 9.17
& \underline{50.21}
& 4.17
& 3.33
& \underline{16.32}\\

\rowcolor{oursgray}
\ours{}
& \textbf{15.24}
  \textcolor{improvegreen}{\scriptsize \(+0.52\)}
& \textbf{9.59}
  \textcolor{improvegreen}{\scriptsize \(+0.28\)}
& \textbf{52.04}
  \textcolor{improvegreen}{\scriptsize \(+1.83\)}
& \textbf{4.72}
  \textcolor{improvegreen}{\scriptsize \(+0.28\)}
& 4.44\textcolor{improvegreen}{\scriptsize \(-0.84\)}
& \textbf{17.21}
  \textcolor{improvegreen}{\scriptsize \(+0.89\)}\\

\midrule

\multicolumn{7}{@{}l}{%
\textbf{Qwen3-4B student} /
\textit{Qwen3-8B-Math-GRPO teacher}}\\
\addlinespace[2pt]

Initial student \citep{yang2025qwen3}
& 20.00
& 18.33
& 64.58
& --
& --
& --\\

Vanilla-GRPO \citep{shao2024deepseekmath}
& 66.67
& 57.22
& 94.79
& 36.67
& 44.44
& 59.96\\

Vanilla-PG-OPD \citep{agarwal2024gkd}
& 65.56
& 53.89
& 94.17
& 31.94
& 42.50
& 57.61\\

Naive-GRPO+PG-OPD
& 65.56
& 55.83
& 93.54
& 33.61
& 44.72
& 58.65\\

Distilled RL \citep{wang2026distilledrl}
& 68.33
& \textbf{67.50}
& 97.08
& \underline{40.83}
& \textbf{48.89}
& \underline{64.53}\\

ATOD-aligned \citep{tan2026atod}
& \underline{71.11}
& 63.61
& \textbf{97.92}
& 40.28
& 48.06
& 64.20\\

\rowcolor{oursgray}
\ours{}
& \textbf{71.39}
  \textcolor{improvegreen}{\scriptsize \(+0.28\)}
& \underline{66.94}\textcolor{improvegreen}{\scriptsize \(-0.56\)}
& \underline{97.78}\textcolor{improvegreen}{\scriptsize \(-0.14\)}
& \textbf{41.01}
  \textcolor{improvegreen}{\scriptsize \(+0.18\)}
& \underline{48.33}\textcolor{improvegreen}{\scriptsize \(-0.56\)}
& \textbf{65.09}
  \textcolor{improvegreen}{\scriptsize \(+0.56\)}\\

\bottomrule
\end{tabular}%
}
\end{table}

%% file: content/experiments.tex
\section{Experiments}
\label{sec:experiments}

We evaluate whether PUU improves trajectory ranking and whether ECR further
improves token credit assignment. Main comparisons use a Qwen3-1.7B student
with a Qwen3-4B-GRPO teacher and a Qwen3-4B student with a
Qwen3-8B-Math-GRPO teacher. Component ablations use only the Qwen3-1.7B
student and Qwen3-4B-GRPO teacher.

\subsection{Experimental Setup}


\paragraph{Models and training.}
The 1.7B setting initializes Qwen3-1.7B-Base and trains it for 515 steps on
difficulty-5--7 problems from DeepMath-103K, using a frozen Qwen3-4B-GRPO
teacher and \(G=8\) responses per prompt \citep{yang2025qwen3}. The larger
setting pairs a Qwen3-4B student with a frozen Qwen3-8B-Math-GRPO teacher and
trains for 160 steps on a sampled difficulty-6--8 subset. Both models use
non-thinking mode, with a global prompt batch size of 126 and \(G=8\).
Within each scale, all methods share the data, decoding, optimizer, clipping,
and reference-KL settings. Rollouts and PPO importance ratios use
\(T_r=0.7\), while teacher and old-policy scores in the teacher gap use raw
\(T=1\). Appendices~\ref{app:experiments} and~\ref{app:actor-update} provide
the remaining settings and actor-loss conventions.

\paragraph{Baselines.}
We compare Vanilla GRPO \citep{shao2024deepseekmath,mroueh2025grpo},
Vanilla PG-OPD \citep{agarwal2024gkd}, a naive sum of independently clipped
GRPO and PG-OPD losses, Distilled RL \citep{wang2026distilledrl}, and an
aligned single-response adaptation of ATOD \citep{tan2026atod}. All
teacher-assisted methods use the same frozen scorer within each model scale.

\paragraph{Evaluation.}
We report \(\mathrm{Avg@12}\) accuracy on AIME 2024, AIME 2025, AMC 2023,
and the February and November 2025 HMMT contests under a fixed 12-sample
decoding setup. The average weights all five benchmarks equally.

\subsection{Main Results}
\paragraph{Results.}
Table~\ref{tab:main} shows that our method achieves the highest
five-benchmark average at both model scales. For the 1.7B student, it leads
on four benchmarks, exceeding the strongest baseline by 0.52, 0.28, 1.83,
and 0.28 points on AIME24, AIME25, AMC23, and HMMT25-Feb, respectively.
Vanilla GRPO leads on HMMT25-Nov by 0.84 points. Our average reaches 17.21\%,
surpassing ATOD-aligned by 0.89 points. For the 4B student, our method leads
on AIME24 and HMMT25-Feb by 0.28 and 0.18 points, while trailing the best
baseline on AIME25, AMC23, and HMMT25-Nov by 0.56, 0.14, and 0.56 points.
Its average of 65.09\% nevertheless exceeds Distilled RL by 0.56 points,
showing that the aggregate gain does not imply uniform improvement. Training dynamics for both model scales are provided in
Appendix~\ref{app:experiments}. In the 1.7B runs
(Figure~\ref{fig:qwen17b-training}), our method achieves higher evaluation
accuracy than pure OPD and the naive hybrid without ATOD's large initial
gradient norm; the corresponding 4B curves appear in
Figure~\ref{fig:qwen4b-training}.

\begin{figure*}[t]
\centering
\includegraphics[width=0.97\textwidth,trim=0 50bp 0 7bp,clip]
{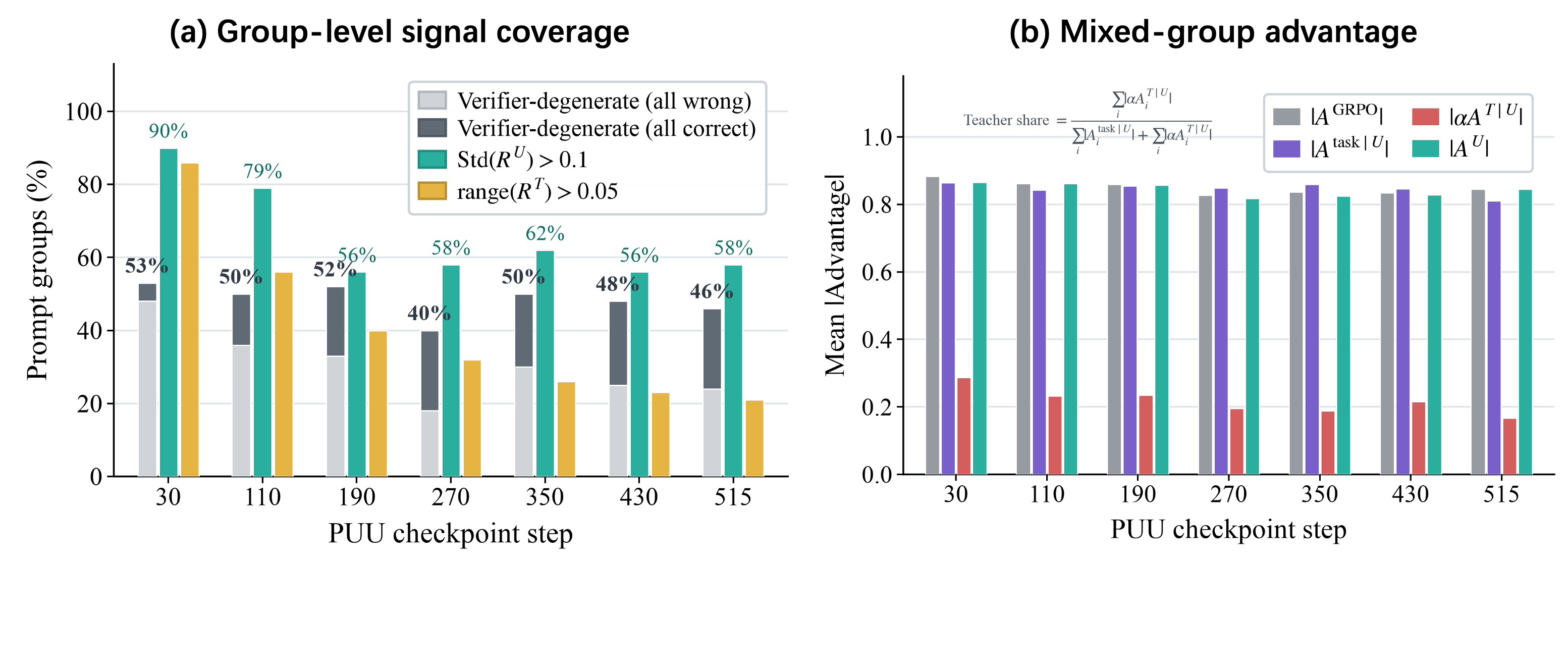}
\caption{PUU signal analysis on AMC 2023. Twelve responses to each of
40 problems form pseudo-groups of eight and four. (a) Fractions of
verifier-degenerate groups and groups with
\(\operatorname{Std}(R^U)>0.1\) or
\(\operatorname{range}(R^T)>0.05\).
(b) Mean absolute PUU advantage components in mixed groups.}
\label{fig:puu-group-signal}
\end{figure*}

\subsection{Ablation Studies}

We isolate trajectory construction from token redistribution with the
Qwen3-1.7B student and frozen Qwen3-4B-GRPO teacher
(Table~\ref{tab:component-ablation}). All training variants share the data,
initialization, rollout, and optimizer settings, with \(\alpha=1.0\).

\paragraph{Trajectory-level normalization.}
With ECR disabled, Task only uses verifier reward, Separate norm. normalizes
task and teacher advantages independently, and PUU instead normalizes their
joint response utility. Figure~\ref{fig:puu-group-signal} examines the
resulting signals offline. Verifier rewards are tied within 40--53\% of the
pseudo-groups, whereas 56--90\% of all groups retain nontrivial variation in
the unified utility. In mixed groups, the teacher component accounts for
17--25\% of the summed absolute component magnitude, while the unified
advantage remains close in scale to verifier-only GRPO. PUU therefore adds
response-level resolution without allowing the teacher term to dominate the
task component in this audit. These magnitudes are neither gradient nor
performance contributions. Appendix~\ref{app:puu-signal-audit} gives the full
protocol and discusses its offline construction. A separate fixed-rollout
\(\alpha\) sweep finds no correct--incorrect ordering inversions at the trained
setting \(\alpha=1\) over 451 mixed AMC 2023 groups and 22,811 response pairs;
this is an empirical observation on held-out rollouts, not a guarantee or a
retraining-based sensitivity result (Appendix~\ref{app:alpha-sweep}).

\paragraph{Token-level redistribution.}
Holding PUU fixed, the lower block of Table~\ref{tab:component-ablation}
shows that Full ECR outperforms both component removals.
Figure~\ref{fig:ecr-token-localization} provides a separate
offline diagnosis of where ECR places task credit. On failed responses from
mixed verifier groups, blinded Opus 4.8 labels identify valid and erroneous
reasoning steps. Full ECR yields a positive invalid-minus-valid task-weight
gap and exceeds all 100 controls that shuffle entropy positions
within each response. Thus, ECR relatively protects judge-valid reasoning
while concentrating negative task credit on identified errors, and its
zero-sum projection preserves the response-wise budget. This establishes the
intended localization behavior rather than a causal accuracy gain; the
protocol and statistics appear in Appendix~\ref{app:ecr-localization}.

\begin{figure*}[t]
\centering
\includegraphics[width=\textwidth,trim=0 0 0 7bp,clip]
{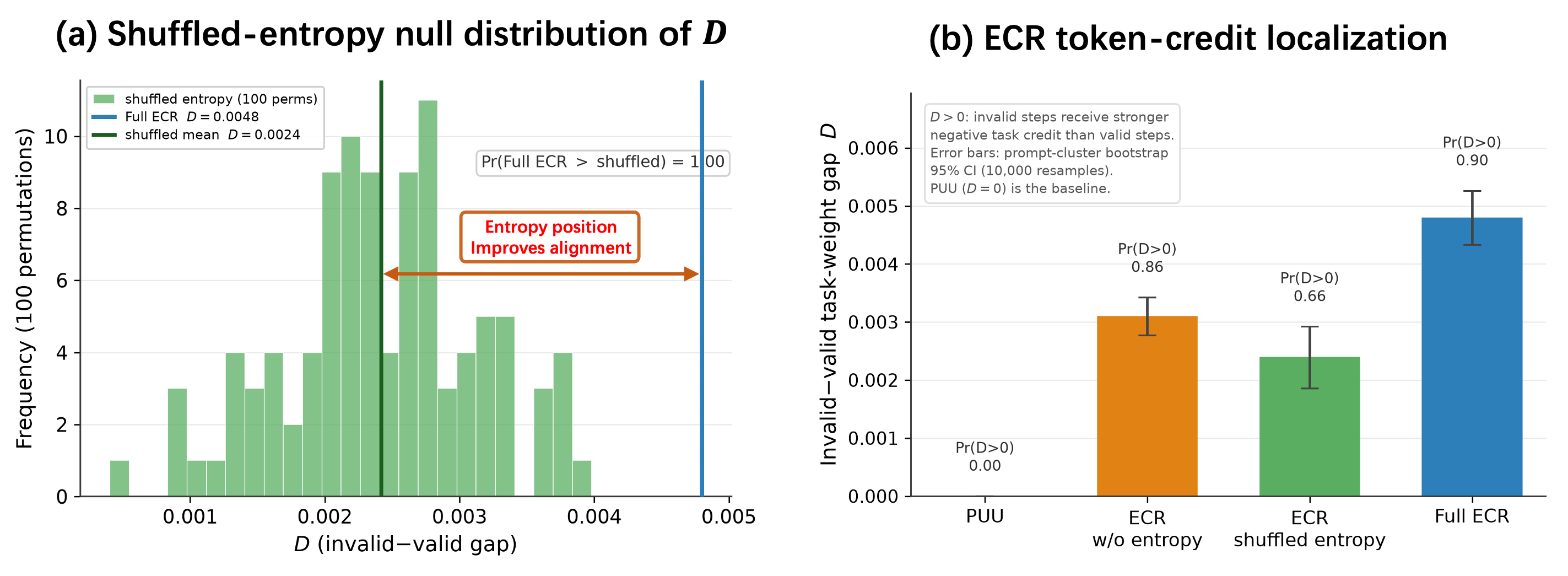}
\caption{Offline ECR token-credit localization at initialization.
(a) Invalid-minus-valid weight gaps under 100 within-response entropy
shuffles; the original entropy positions outperform all shuffles.
(b) Gaps for PUU and three ECR variants with 95\% prompt-cluster bootstrap
intervals. Labels give the fraction of responses with \(D_i>0\), indicating
stronger negative credit on invalid steps.}
\label{fig:ecr-token-localization}
\end{figure*}

\begin{table}[t]
\centering
\scriptsize
\setlength{\tabcolsep}{2.5pt}
\setlength{\belowcaptionskip}{5pt}
\renewcommand{\arraystretch}{1.12}
\caption{Internal ablations on Qwen3-1.7B with a frozen Qwen3-4B-GRPO
teacher. The first row group varies trajectory construction with ECR disabled;
the second fixes the PUU trajectory advantage and varies ECR token-credit
components. The \emph{+ PUU} row (Avg.\ 20.28) is the no-ECR baseline that the
Token-credit rows build upon (denoted \emph{PUU only $+$}). \(\checkmark\)/\(\times\)
mark which components are active.}
\label{tab:component-ablation}
\begin{tabular}{@{}l ccc ccc cccc@{}}
\toprule
\multirow{2}{*}{Variant}
& \multicolumn{3}{c}{Trajectory comp.}
& \multicolumn{3}{c}{ECR comp.}
& \multicolumn{4}{c}{Avg@12 (\%)}\\
\cmidrule(lr){2-4}\cmidrule(lr){5-7}\cmidrule(lr){8-11}
& Task & Teacher & Joint norm. & Gap & Entropy & Proj. & AIME24 & AIME25 & AMC23 & Avg.\\
\midrule
\multicolumn{11}{@{}l}{\textbf{Trajectory-level} \quad \textit{ECR disabled}}\\
Task only               & \(\checkmark\) & \(\times\)     & \(\times\)     & \(\times\) & \(\times\) & \(\times\) & 7.78  & 6.39  & 37.50 & 17.22\\
+ Separate norm.        & \(\checkmark\) & \(\checkmark\) & \(\times\)     & \(\times\) & \(\times\) & \(\times\) & 8.50  & 7.13  & 38.94 & 18.19\\
\rowcolor{oursgray}+ PUU & \(\checkmark\) & \(\checkmark\) & \(\checkmark\) & \(\times\) & \(\times\) & \(\times\) & 12.32 & 8.02  & 40.50 & 20.28\\
\midrule
\multicolumn{11}{@{}l}{\textbf{Token-credit} \quad \textit{PUU baseline $+$ ECR variants}}\\
PUU only $+$ w/o entropy    & \(\checkmark\) & \(\checkmark\) & \(\checkmark\) & \(\checkmark\) & \(\times\)     & \(\checkmark\) & 12.64 & 8.24  & 42.46 & 21.11\\
PUU only $+$ w/o projection & \(\checkmark\) & \(\checkmark\) & \(\checkmark\) & \(\checkmark\) & \(\checkmark\) & \(\times\)     & 15.04 & 9.17  & 51.32 & 25.17\\
\rowcolor{oursgray}Full \ours{} & \(\checkmark\) & \(\checkmark\) & \(\checkmark\) & \(\checkmark\) & \(\checkmark\) & \(\checkmark\) & 15.24 & 9.59 & 52.04 & 25.62\\
\bottomrule
\end{tabular}
\end{table}

%% file: content/conclusion.tex
\section{Conclusion}

UECR-GRPO separates trajectory ranking from token credit assignment. PUU
combines verifier and teacher utility before group normalization, while ECR
uses signed teacher gaps and entropy to redistribute verifier-derived credit
without changing each response's task-credit budget. Experiments across five benchmarks and two Qwen3 model scales show that
UECR-GRPO achieves higher average accuracy than the compared training
objectives, while offline diagnostics support the intended group-level
resolution and token localization.

%% file: content/appendix.tex
\section{Clipped Actor Update and Temperature Semantics}
\label{app:actor-update}

Let \(m_{i,t}\) select valid response tokens and let \(A_{i,t}\) be any
detached token advantage.  The current-to-old ratio and clipped policy
surrogate are
\begin{equation}
\begin{aligned}
r_{i,t}(\theta)
&=\exp\!\left(
\log\pi_{\theta,T_r}(y_{i,t}\mid s_{i,t})
-\log\pi_{\old,T_r}(y_{i,t}\mid s_{i,t})
\right),\\
\mathcal L_{\mathrm{PG}}(A)
&=-\operatorname{Reduce}_{m_{i,t}=1}
\min\!\left[
r_{i,t}(\theta)A_{i,t},
\clip(r_{i,t}(\theta),1-\epsilon_\ell,1+\epsilon_h)A_{i,t}
\right].
\end{aligned}
\label{eq:policy-loss}
\end{equation}
Both actor log probabilities use the rollout convention \(T_r=0.7\).  The
teacher gap instead compares \(\pi_T\) and \(\pi_\old\) at raw \(T=1\).
Prompt tokens, response padding, and synthetic data-parallel padding rows are
excluded.  The reference-policy KL, when enabled, is added exactly once to
the actor loss and is not included in task reward, teacher reward, or the PUU
group statistics.

\section{Path-Space Target}
\label{app:path-theory}

\subsection{Setup and support assumptions}

Fix a prompt \(x\) and suppress it where unambiguous.  Let \(\mathcal Y\) be
the set of finite response paths.  The anchor \(P_0\), teacher \(P_T\), and
candidate distribution \(Q\) are causal autoregressive path measures.  We
assume \(P_T(y)>0\) and \(Q(y)>0\) only where \(P_0(y)>0\), \(\beta>0\), and a
finite partition function.  These absolute-continuity conditions ensure that
the density ratios and KL divergences are defined.

\begin{lemma}[Token log-ratios telescope to a path log-ratio]
\label{lem:path-ratio}
For any \(y=(y_1,\ldots,y_L)\),
\[
\sum_{t=1}^{L}
\log\frac{\pi_T(y_t\mid x,y_{<t})}
{\pi_0(y_t\mid x,y_{<t})}
=\log\frac{P_T(y\mid x)}{P_0(y\mid x)}.
\]
\end{lemma}

\begin{proof}
Autoregressive factorization gives
\(P_T(y\mid x)=\prod_t\pi_T(y_t\mid x,y_{<t})\) and the analogous identity
for \(P_0\).  Dividing the products and taking a logarithm gives the result.
\end{proof}

Define
\[
U(y)=R_\task(y)+\alpha\log\frac{P_T(y)}{P_0(y)}
\]
and consider
\begin{equation}
\max_Q\left\{\E_{y\sim Q}[U(y)]-\beta\KL(Q\|P_0)\right\}.
\label{eq:appendix-control-objective}
\end{equation}

\begin{proposition}[Gibbs-optimal path distribution]
\label{prop:gibbs}
Let
\[
Z=\E_{y\sim P_0}\!\left[\exp\!\left(\frac{U(y)}{\beta}\right)\right].
\]
The unique maximizer of Eq.~\eqref{eq:appendix-control-objective} is
\[
Q^*(y)
=\frac{1}{Z}P_0(y)\exp\!\left(\frac{U(y)}{\beta}\right)
=\frac{1}{Z}\exp\!\left(\frac{R_\task(y)}{\beta}\right)
P_T(y)^{\alpha/\beta}P_0(y)^{1-\alpha/\beta}.
\]
\end{proposition}

\begin{proof}
The first expression is normalized by \(Z\), and
\[
\log\frac{Q^*(y)}{P_0(y)}=\frac{U(y)}{\beta}-\log Z.
\]
For any feasible \(Q\),
\begin{align*}
\KL(Q\|Q^*)
&=\E_Q\!\left[
\log\frac{Q(y)}{P_0(y)}-\frac{U(y)}{\beta}+\log Z
\right],\\
\E_Q[U]-\beta\KL(Q\|P_0)
&=\beta\log Z-\beta\KL(Q\|Q^*).
\end{align*}
Non-negativity of KL gives a unique optimum at \(Q=Q^*\) almost everywhere.
The second form follows by substituting the verifier and teacher terms of \(U\).
\end{proof}

\begin{corollary}[Target-level reductions]
\label{cor:special-cases}
The optimum has the following reductions:
\begin{enumerate}[leftmargin=*,topsep=2pt,itemsep=1pt]
  \item If \(\alpha=0\), then
  \(Q^*\propto P_0\exp(R_\task/\beta)\), the standard
  KL-regularized reward tilt.
  \item If \(R_\task=0\) and \(\alpha=\beta\), then \(Q^*=P_T\).
  \item If \(R_\task=0\) and \(\alpha>\beta\), then
  \(Q^*\propto P_T^{\alpha/\beta}P_0^{1-\alpha/\beta}\), an extrapolation in
  the teacher-over-anchor reward direction.
  \item With both terms active, task reward and teacher preference tilt the same target.
\end{enumerate}
\end{corollary}

\paragraph{Relationship to the implementation.}
The theorem is an organizing target, not a claim that every implementation
detail exactly optimizes Eq.~\eqref{eq:appendix-control-objective}.  First,
the identity in Lemma~\ref{lem:path-ratio} uses a token sum, whereas the
algorithm uses a masked mean to control reward--length correlation.  For
variable-length responses this changes the utility.  Second, the
iteration-wise anchor is \(P_\old\), held fixed while a batch is scored and
updated.  Third, PPO clipping is the practical proximal mechanism; an
optional long-horizon reference KL is an additional actor regularizer and is
not part of the teacher reward.

\section{Exact Decomposition of the PUU Advantage}
\label{app:decomposition}

Within one prompt group, let
\[
R_i^U=R_i^\task+\alpha R_i^T,\quad
\mu_U=G^{-1}\sum_iR_i^U,\quad
\sigma_U=\operatorname{Std}_{i=1}^{G}(R_i^U).
\]
Define
\[
A_i^{\task\mid U}=\frac{R_i^\task-\mu_\task}{\sigma_U+\epsilon},\qquad
A_i^{T\mid U}=\frac{R_i^T-\mu_T}{\sigma_U+\epsilon},\qquad
A_i^U=\frac{R_i^U-\mu_U}{\sigma_U+\epsilon}.
\]

\begin{proposition}[Advantage decomposition]
\label{prop:adv-decomposition}
Using the shared PUU denominator,
\[
A_i^U=A_i^{\task\mid U}+\alpha A_i^{T\mid U}.
\]
\end{proposition}

\begin{proof}
Linearity gives \(\mu_U=\mu_\task+\alpha\mu_T\), hence
\[
R_i^U-\mu_U=(R_i^\task-\mu_\task)
+\alpha(R_i^T-\mu_T).
\]
Division by the common denominator proves the identity.
\end{proof}

If the two rewards were separately standardized by different standard
deviations, this identity would fail and the effective teacher coefficient
would vary with each group's empirical scales.

\section{Constrained Token-Credit Projection}
\label{app:projection}

For one non-empty response, let the raw local proposal be
\(\frac12c_{i,t}d_{i,t}\).  We seek its closest zero-sum adjustment:
\begin{equation}
\min_{\{q_{i,t}\}}
\sum_{t:m_{i,t}=1}
\frac{\left(q_{i,t}-\frac12c_{i,t}d_{i,t}\right)^2}{c_{i,t}}
\quad\text{subject to}\quad
\sum_{t:m_{i,t}=1}q_{i,t}=0.
\label{eq:appendix-projection}
\end{equation}
Because \(c_{i,t}=\exp(-H^T_{i,t}/\tau_H)>0\) for finite entropy, the
objective is strictly convex.

\begin{proposition}[Closed-form projection]
\label{prop:projection}
The unique solution is
\[
q_{i,t}=\frac12c_{i,t}(d_{i,t}-\mu_i^c),\qquad
\mu_i^c=\frac{\sum_t m_{i,t}c_{i,t}d_{i,t}}
{\sum_t m_{i,t}c_{i,t}}.
\]
\end{proposition}

\begin{proof}
Introduce a multiplier \(\lambda\) for the zero-sum constraint.  Stationarity
at a valid token gives
\[
\frac{2}{c_{i,t}}
\left(q_{i,t}-\tfrac12c_{i,t}d_{i,t}\right)+\lambda=0,
\]
so \(q_{i,t}=\frac12c_{i,t}(d_{i,t}-\lambda)\).  The constraint yields
\(\lambda=\mu_i^c\).  Strict convexity gives uniqueness.
\end{proof}

\begin{proposition}[Additive task-budget and sign preservation]
\label{prop:invariants}
For every non-empty response and \(0\le\rho<1\),
\[
\frac1{L_i}\sum_{t:m_{i,t}=1}
A_i^{\task\mid U}(1+\rho q_{i,t})
=A_i^{\task\mid U},
\]
and every nonzero token-wise task component retains the sign of
\(A_i^{\task\mid U}\).
\end{proposition}

\begin{proof}
The definition of \(\mu_i^c\) implies
\(\sum_{t:m_{i,t}=1}q_{i,t}=0\).  Since \(d_{i,t}\in[-1,1]\), its
confidence-weighted mean \(\mu_i^c\in[-1,1]\).  Together with
\(c_{i,t}\in(0,1]\), this gives \(|q_{i,t}|\le1\), hence
\(1-\rho\le1+\rho q_{i,t}\le1+\rho\).  The weight is positive and has
arithmetic mean one, proving both claims.
\end{proof}

The proposition concerns only the verifier-derived component before clipping.
It does not fix the full PUU advantage sign or preserve the gradient norm
and parameter update after clipping.

\subsection{Entropy computation and limiting cases}

For raw teacher logits \(z^T_{i,t}\), entropy is evaluated in float32 as
\[
H^T_{i,t}
=\log\sum_v\exp z^T_{i,t,v}
-\sum_v p^T_{i,t,v}z^T_{i,t,v},
\qquad p^T_{i,t}=\softmax(z^T_{i,t}).
\]
As \(H^T_{i,t}\) grows, \(c_{i,t}\) decreases and the local proposal is
attenuated.  If \(\rho=0\), confidence vanishes at every valid token, or all
directions within a response are identical, the redistribution term vanishes
and every token receives the original PUU advantage.

\section{Direct Comparison with Closest Methods}
\label{app:closest-methods}

\subsection{Distilled RL}

For a response-level GRPO advantage \(A_i\), Distilled RL defines
\[
\rho^T_{i,t}=\frac{\pi_T(y_{i,t}\mid s_{i,t})}
{\pi_\old(y_{i,t}\mid s_{i,t})},\qquad
\bar\rho^T_{i,t}=\clip(\rho^T_{i,t},\epsilon_T^{-1},\epsilon_T),
\]
\[
\widetilde\rho^T_{i,t}
=\frac{\bar\rho^T_{i,t}}
{\exp\left(L_i^{-1}\sum_s\log\bar\rho^T_{i,s}\right)},\qquad
w^{\mathrm{DRL}}_{i,t}=
\begin{cases}
\widetilde\rho^T_{i,t},&A_i>0,\\
1,&A_i\le0.
\end{cases}
\]
It then inserts \(w^{\mathrm{DRL}}_{i,t}A_i\) into the clipped policy
surrogate \citep{wang2026distilledrl}.  The teacher ratio is thus applied
after verifier group normalization and cannot alter response ranking.  Its
geometric normalization guarantees
\((\prod_t\widetilde\rho^T_{i,t})^{1/L_i}=1\), not
\(L_i^{-1}\sum_t\widetilde\rho^T_{i,t}=1\).  Our no-projection audit measures
the practical difference between these invariants under token-mean reduction.

\subsection{ATOD}

ATOD constructs
\[
A_{i,t}^{\mathrm{ATOD}}
=\kappa(s)\,w_{k(t)}\Delta\log p_{i,t}
+\rho(s)A_i^{\mathrm{GRPO}},
\]
where \(\kappa(s)\) decreases, \(\rho(s)\) increases, and \(w_{k(t)}\) is the
T-DUR weight for the containing interaction turn
\citep{tan2026atod}.  T-DUR combines normalized turn-level
teacher--student disagreement with student sampled-token surprisal and weights
only the OPD term.  The sum is optimized through one clipped actor surrogate.
ATOD is therefore not an independent-loss method.  It differs from PUU because
the GRPO advantage is normalized before combination and the teacher term is a
token-level OPD advantage rather than part of the response utility used for
group ranking.

\begin{table*}[t]
\centering
\footnotesize
\setlength{\tabcolsep}{3pt}
\renewcommand{\arraystretch}{1.15}
\caption{Conceptual comparison with the closest baselines.  ``Exact additive
budget'' refers to the arithmetic token mean of the verifier-derived
component before PPO clipping.}
\label{tab:method-comparison}
\begin{tabular}{@{}>{\raggedright\arraybackslash}p{1.02in}*{2}{>{\raggedright\arraybackslash}p{0.63in}}>{\raggedright\arraybackslash}p{1.06in}>{\raggedright\arraybackslash}p{0.65in}>{\raggedright\arraybackslash}p{0.72in}@{}}
\toprule
Method & Teacher in ranking & One clipped update
& Negative task credit & Entropy attenuation
& Exact credit budget\\
\midrule
Naive GRPO+OPD & No & No & OPD independent & No & No\\
ATOD & No & Yes & Through additive OPD & No & No\\
Distilled RL & No & Yes & Reset to GRPO & No & No\\
PUU & Yes & Yes & Yes, at trajectory level & No & Broadcast\\
\ours{} & Yes & Yes & Yes, signed protection & Yes & Yes\\
\bottomrule
\end{tabular}
\end{table*}

\section{Additional Experimental Protocol}
\label{app:experiments}

\subsection{Signal diagnosis across training}
\label{app:signal-diagnosis}

For each checkpoint and response, the verifier returns
\(R_i^\task\in\{0,1\}\), while the teacher score is
\[
R_i^T=\frac{\sum_t m_{i,t}
\left(\log\pi_{T,T=1}(y_{i,t}\mid s_{i,t})
-\log\pi_{\old,T=1}(y_{i,t}\mid s_{i,t})\right)}
{\sum_t m_{i,t}}.
\]
A verifier-degenerate group contains identical binary rewards and hence zero
within-group task advantage.  We report (i) the fraction of such groups,
(ii) the fraction in which \(R_i^T\) has nonzero variation, (iii) the
teacher--verifier ranking-conflict rate in non-degenerate groups, and (iv) a
blinded model audit on verifier-tied pairs.  The initialization tie rate is 41.3\%, as shown in
Figure~\ref{fig:signal-complementarity}.

\subsection{Token-credit diagnosis at initialization}
\label{app:token-credit-diagnosis}

The mechanism audit uses the shared initial 1.7B student and frozen 4B teacher
before training.  Eight responses are sampled for each of 300 prompts, giving
2,400 trajectories and 21.55M valid response tokens.  No optimizer update or
test benchmark is involved.  For each valid token,
\[
\bar d_{i,t}=\tanh\!\left(\frac{\delta_{i,t}}{2\tau_\delta}\right),
\quad
d_{i,t}=\operatorname{sign}(A_i^{\task\mid U})\bar d_{i,t},
\quad
c_{i,t}=\exp(-H^T_{i,t}/\tau_H).
\]
Entropy strata are determined once from global step-0 quantiles.  Tied
near-deterministic values may collapse repeated boundaries, so these are
called strata rather than strict deciles.  We report token-weighted and
response-balanced summaries with prompt-cluster bootstrap.

Teacher concentration is measured by realized-token probability and
realized-token/top-1 match, neither of which is called token correctness.
Large-gap concentration is reported at top 5\%, 10\%, and 20\% thresholds.
The primary audit finds that 82.4\% of top-10\% gaps occur in the highest
entropy quartile and that calibration reduces their mean magnitude by 37\%.

For projected and unprojected weights, define
\[
B_i=\frac{\sum_t m_{i,t}w_{i,t}}{\sum_t m_{i,t}}.
\]
The projection gives \(B_i=1\) analytically and a maximum observed numerical
error of \(2.22\times10^{-16}\).  The unprojected variant ranges from 0.98
to 1.02.  These quantities measure task-credit mass before clipping; they do not
measure the resulting gradient norm or parameter update.

\subsection{Offline PUU group-signal audit}
\label{app:puu-signal-audit}

\paragraph{Rollouts and scoring.}
We evaluate all 40 AMC 2023 problems using Qwen3-1.7B checkpoints from PUU
training. Each checkpoint samples 12 responses per problem with temperature
0.6, top-\(p=0.95\), top-\(k=20\), and seed 42. The verifier assigns
\(R_i^\task\in\{0,1\}\) by exact integer match with the reference answer. The
frozen Qwen3-4B-GRPO teacher and Qwen3-1.7B-Base anchor, denoted by
\(\pi_T\) and \(\pi_B\), score the same realized response tokens at raw
temperature 1. The offline teacher score is
\begin{equation}
R_i^T=\frac{1}{L_i}\sum_{t=1}^{L_i}
\clip\!\left(
\log\pi_T(y_{i,t}\mid s_{i,t})-
\log\pi_B(y_{i,t}\mid s_{i,t}),-5,5
\right),
\qquad R_i^U=R_i^\task+R_i^T,
\label{eq:app-puu-audit-reward}
\end{equation}
which corresponds to \(\alpha=1\) in this audit. For each problem, a fixed
seed partitions the 12 responses into a group of eight and a retained group
of four, producing 80 pseudo-groups per checkpoint. These groups are used only
for the diagnostic and do not alter training.

\paragraph{Group-level signal coverage.}
Panel (a) of Figure~\ref{fig:puu-group-signal} uses checkpoints 30, 110, 190,
270, 350, 430, and 515. A group is verifier-degenerate when all its task
rewards are identical. We report its all-wrong and all-correct portions, the
fraction of all groups with \(\operatorname{Std}(R^U)>0.1\), and the fraction
with \(\max_iR_i^T-\min_iR_i^T>0.05\). The latter two statistics are computed
over all groups rather than conditioned on verifier degeneracy. Across the
selected checkpoints, verifier-degenerate groups account for 40--53\% of the
total, while nontrivial unified-utility variation appears in 56--90\%. The
teacher-range statistic decreases from 86\% at step 30 to 21\% at step 515,
showing that its additional response-level variation is strongest early in
training.

\paragraph{Mixed-group advantage composition.}
Panel (b) uses checkpoints 30, 60, 90, 120, 180, 210, and 230 and retains only
groups containing both verifier outcomes. We compare the verifier-only GRPO
advantage with the exact PUU decomposition
\[
A_i^U=A_i^{\task\mid U}+\alpha A_i^{T\mid U},
\]
where the two components share the unified denominator. Each plotted value is
the response-mean absolute magnitude at that checkpoint. We additionally
compute
\[
\operatorname{TeacherShare}=
\frac{\sum_i|\alpha A_i^{T\mid U}|}
{\sum_i|A_i^{\task\mid U}|+\sum_i|\alpha A_i^{T\mid U}|}.
\]
The teacher share ranges from 17\% to 25\%, and \(|A^U|\) remains close in
mean magnitude to \(|A^{\mathrm{GRPO}}|\). Thus the teacher term is visible
without overwhelming the task component under this construction.

This is a signal-geometry audit rather than a training ablation. It uses
pseudo-groups and a fixed base-model anchor, whereas the practical on-policy
update uses the behavior policy as its current anchor. The plotted component
magnitudes therefore do not estimate parameter-gradient or accuracy
contributions; the controlled training variants are reported in
Table~\ref{tab:component-ablation}.

\subsection{Offline sensitivity of verifier-consistent ordering to \(\alpha\)}
\label{app:alpha-sweep}

\paragraph{Scope and fixed rollouts.}
We perform a fixed-rollout sensitivity audit of the PUU teacher coefficient
\(\alpha\); this is not a training-performance ablation. The audit uses 17
Qwen3-1.7B checkpoints from ATOD training (steps 30, 60, 90, through 480, and
515) with the frozen Qwen3-4B-GRPO teacher. At each checkpoint, we reuse
previously generated responses on the held-out AMC 2023 benchmark: 24 samples
per problem through step 240 and 12 samples thereafter, for 40 problems per
checkpoint. The resulting data contain 451 mixed verifier groups and 22,811
correct--incorrect response pairs. These are benchmark groups rather than
training-prompt groups, and their sizes differ from the training group size
\(G=8\).

\paragraph{Checkpoint-relative teacher score.}
For each response, the teacher and the student checkpoint that generated that
response are evaluated at raw temperature 1. We compute
\begin{equation}
R_i^T=\frac{1}{L_i}\sum_{t=1}^{L_i}
\clip\!\left(
\log\pi_T(y_{i,t}\mid s_{i,t})-
\log\pi_{\old,i}(y_{i,t}\mid s_{i,t}),-5,5
\right),
\label{eq:app-alpha-audit-teacher-score}
\end{equation}
where \(\pi_{\old,i}\) is the generating checkpoint, not a fixed reference.
The masked token mean, per-token clipping, and exclusion of prompt and padding
tokens match the training reward configuration. Log probabilities are obtained
once with vLLM prompt scoring and cached before the CPU-only \(\alpha\) sweep.
Scoring uses a maximum model length of 8,192 tokens. Responses beyond this
limit, accounting for approximately 10--24\% depending on the checkpoint, are
head-truncated; their \(R_i^T\) therefore uses only the available response
tokens. Truncation and non-finite-token counts are recorded rather than
silently discarded.

\paragraph{Ordering metrics.}
For a mixed group \(g\), let \(\mathcal C_g\) and \(\mathcal W_g\) denote its
correct and incorrect responses. At coefficient \(\alpha\),
\begin{equation}
U_c(\alpha)=1+\alpha R_c^T,\qquad
U_w(\alpha)=\alpha R_w^T.
\label{eq:app-alpha-audit-utility}
\end{equation}
Group normalization is order preserving, so comparing these utilities is
equivalent to comparing their PUU advantages. With tie tolerance \(10^{-6}\),
the primary conflict statistic first averages strict inversions within each
group and then gives every mixed group equal weight:
\begin{equation}
C_g(\alpha)=\frac{1}{|\mathcal C_g||\mathcal W_g|}
\sum_{c\in\mathcal C_g}\sum_{w\in\mathcal W_g}
\mathbf 1\!\left[U_w(\alpha)>U_c(\alpha)+10^{-6}\right].
\label{eq:app-alpha-audit-conflict}
\end{equation}
We also report a pair-micro average as a secondary statistic. A group is fully
safe at \(\alpha\) when
\(M_g(\alpha)=\min_{c\in\mathcal C_g}U_c(\alpha)
-\max_{w\in\mathcal W_g}U_w(\alpha)>10^{-6}\). Its first analytical tie point
is
\begin{equation}
\alpha_g^*=\left(
\max_{w\in\mathcal W_g}R_w^T-
\min_{c\in\mathcal C_g}R_c^T
\right)^{-1}
\label{eq:app-alpha-audit-threshold}
\end{equation}
when the parenthesized difference is positive, and \(+\infty\) otherwise.

We separately measure whether the sign of \(A_i^U(\alpha)\) differs from that
of the centered task reward, using population standard deviation and
\(\epsilon=10^{-6}\). This trajectory-sign statistic is not an ordering
statistic: a correct response can remain above every incorrect response yet
fall below the unified group mean. All-correct and all-wrong groups are
excluded from cross-label inversion metrics; for them we report only whether
the teacher score supplies nontrivial within-group variation.

\paragraph{Sweep and uncertainty.}
The pre-specified sweep combines explicit operating points with a dense grid
over \([0,8]\). A second grid places ten points below the first coefficient at
which at least 10\% of mixed groups at a checkpoint contain an inversion. We
use 2,000 bootstrap replicates with seed 42, clustered by AMC problem, for 95\%
intervals. The analytical thresholds are checked against the grid transitions,
and \(\alpha=0\) recovers unit correct--incorrect margins and zero conflicts.
For a zero-event estimate at \(\alpha=1\), we additionally compute the
corresponding 95\% rule-of-three upper bound from the eligible sample count.

\begin{figure*}[t]
\centering
\includegraphics[width=\textwidth]{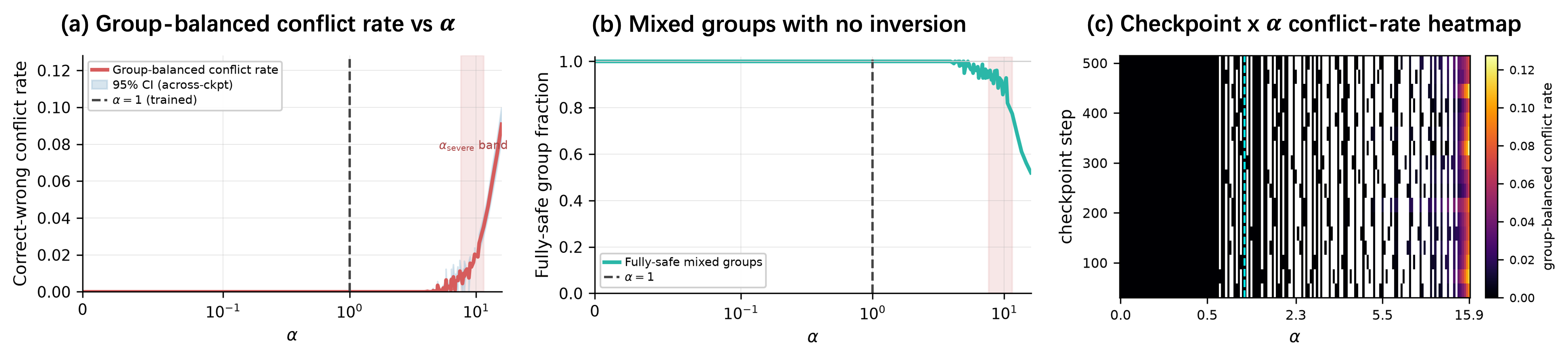}
\caption{Fixed-rollout sensitivity of verifier-consistent ordering to
\(\alpha\) on AMC 2023. The audit covers 17 ATOD-trained Qwen3-1.7B
checkpoints, 451 mixed verifier groups, and 22,811 correct--incorrect pairs.
(a) Group-balanced strict conflict rate, averaged across checkpoints, with a
95\% problem-cluster bootstrap interval. (b) Fraction of mixed groups in which
every correct response remains above every incorrect response. (c)
Checkpoint-wise group-balanced conflict rate. Dashed lines mark the trained
setting \(\alpha=1\); the shaded band spans the checkpoint-wise coefficients
\(7.6\)--\(11.5\) at which at least 10\% of mixed groups first contain an
inversion. This is an offline ordering audit, not a retraining or performance
comparison.}
\label{fig:alpha-sweep}
\end{figure*}

\paragraph{Results and limitations.}
At \(\alpha=1\), we observe no strict correct--incorrect ordering inversion at
any of the 17 checkpoints: both the group-balanced conflict rate and the
any-inversion group rate are zero over the 451 eligible mixed groups and 22,811
pairs. The distinct trajectory-sign statistic ranges from 0 to 3.1\% across
checkpoints. The most restrictive observed group has
\(\alpha_g^*\approx3.88\); the checkpoint-wise 5th percentile of
\(\alpha_g^*\), corresponding to 95\% empirically safe groups, ranges from 5.3
to 9.3, while the median ranges from 12.6 to 21.3. Between 0 and 7.7\% of groups
at a checkpoint have \(\alpha_g^*=+\infty\). Thus \(\alpha=1\) empirically
preserves verifier-consistent correct--incorrect ordering on these observed
rollouts, but this does not constitute a universal guarantee.

The audit is limited to held-out AMC 2023 rollouts rather than the training
prompt distribution; it uses benchmark sample groups of size 12 or 24 rather
than training groups of eight; long responses are truncated at 8,192 tokens;
and \(\pi_{\old,i}\) is the checkpoint that generated each response rather
than the transient old policy inside a training iteration. Finally, the audit
tests cross-label ordering only. It does not establish that teacher-induced
ordering within the correct or incorrect class reflects process quality, nor
does it estimate the performance that retraining at another \(\alpha\) would
produce.

\subsection{Offline ECR credit-localization audit}
\label{app:ecr-localization}

\paragraph{Data and selection.}
This audit uses the shared initialization, with Qwen3-1.7B-Base as both the
student and old policy and the frozen Qwen3-4B-GRPO model as teacher. Prompts
are drawn from the difficulty-5--7 portion of DeepMath-103K after shuffling
with seed 42. The student samples eight responses per prompt at temperature
0.7, top-\(p=1\), and a 16,384-token response limit. A rule-based verifier
compares the extracted \(\boxed{\cdot}\) answer with the reference answer.
Teacher and old-policy log probabilities are evaluated at raw temperature 1,
and teacher entropy uses the full softmax distribution in float32.

We retain real eight-response groups containing both rewards, select only
incorrect responses with \(A_i^{\task\mid U}<0\), and remove empty or
reasoning-free outputs. Selection does not use teacher scores, entropy, ECR
weights, or anticipated labels. This produces 200 candidate responses for
blinded annotation. A deterministic parser divides each response at line, enumeration, and
sentence boundaries. Opus 4.8 sees only the problem, reference answer, and
numbered reasoning steps, and labels each step as valid, invalid, uncertain,
or not reasoning. It also marks at most one first substantive error. Method
signals and model identities remain hidden. Uncertain and non-reasoning steps
are excluded from the primary statistic.

\paragraph{Counterfactual weights.}
For every retained token, we reproduce the training implementation in
Eq.~\eqref{eq:projection-weights}. Four weights are computed on the same
responses: PUU sets \(w_{i,t}=1\); the no-entropy variant sets \(c_{i,t}=1\)
and recomputes the projection; the shuffled control independently permutes
\(H^T_{i,t}\) across valid tokens within each response 100 times, preserving
its multiset while breaking positional alignment; and Full ECR uses the
original entropy and complete zero-sum projection. Each shuffle recomputes
\(c_{i,t}\), \(\mu_i^c\), \(q_{i,t}\), and \(w_{i,t}\).

For response \(i\), the primary localization statistic is
\begin{equation}
D_i=
\underset{t\in\mathcal I_i}{\operatorname{mean}}\,w_{i,t}
-
\underset{t\in\mathcal V_i}{\operatorname{mean}}\,w_{i,t},
\label{eq:ecr-localization-gap}
\end{equation}
where \(\mathcal I_i\) and \(\mathcal V_i\) contain tokens in judge-invalid
and judge-valid steps. Because the analyzed task advantage is negative,
\(D_i>0\) means that invalid steps receive stronger negative task credit than
valid steps. We macro-average \(D_i\) over responses and obtain 95\%
confidence intervals by resampling prompts 10,000 times with seed 42. The
reported \(\Pr(D_i>0)\) is the fraction of eligible responses with a positive
per-response gap, rather than a token-level significance test. A secondary
statistic compares the first erroneous step with valid steps preceding it.
All inference remains clustered by prompt.

\paragraph{Results and scope.}
Full ECR gives \(D=0.0048\) and lies to the right of all 100 entropy-position
permutations, whose mean gap is 0.0024. PUU is exactly zero because it
broadcasts one task weight across the response. The positive Full-ECR gap
shows that its redistribution agrees on average with an independent audit of
step quality, while the permutation result indicates that entropy contributes
positional information beyond its marginal distribution. The projection
preserves a unit response-wise mean weight and hence reallocates rather than
increases the verifier-derived budget. This audit tests semantic alignment of
the realized weights; it neither estimates gradient contribution nor proves a
causal improvement in downstream accuracy.

\subsection{Aligned configurations}

The Qwen3-1.7B student is trained for 515 steps on DeepMath-103K
problems with difficulty levels 5--7 with the Qwen3-4B-GRPO teacher. It uses rollout
temperature 0.7, top-\(p=0.95\),
top-\(k=20\), a 2,048-token prompt limit, a 16,384-token response limit,
learning rate \(10^{-6}\), one PPO epoch, and symmetric clipping 0.2.
Reference KL uses actor-loss \texttt{low\_var\_kl} with coefficient 0.001;
reward KL is disabled.  Realized-token log probabilities and full-vocabulary entropy are computed
from the same teacher logits at raw \(T=1\) in a single chunked scoring pass.

The Qwen3-4B student uses a frozen Qwen3-8B-Math-GRPO teacher and a sampled
subset of DeepMath-103K problems with difficulty levels 6--8. The global prompt batch size is 126, with \(G=8\). Both models
use non-thinking mode in training and evaluation.

\paragraph{Data hygiene and checkpoint selection.}
The draft protocol screens training manifests against AIME 2024, AIME 2025,
and AMC 2023 using normalized exact match, n-gram overlap, and embedding
retrieval followed by manual inspection. The screening results and actual
checkpoint-selection rule still require confirmation from the completed runs.

\begin{figure*}[t]
\centering
\includegraphics[width=\textwidth]{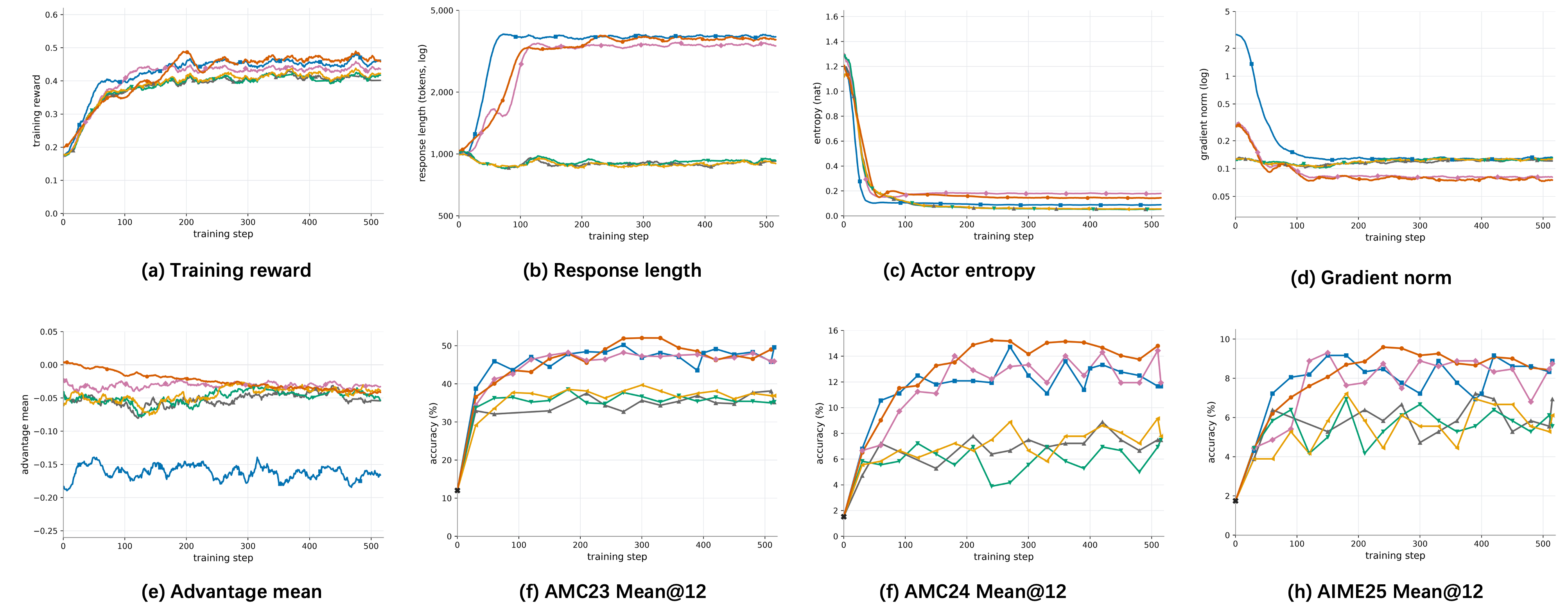}
\caption{Training dynamics for the aligned runs with the Qwen3-1.7B student
and Qwen3-4B-GRPO teacher, including
training reward, response length, actor entropy, gradient norm, advantage
statistics, and fixed-decoding validation accuracy.}
\label{fig:qwen17b-training}
\end{figure*}

\begin{figure*}[t]
\centering
\includegraphics[width=\textwidth]{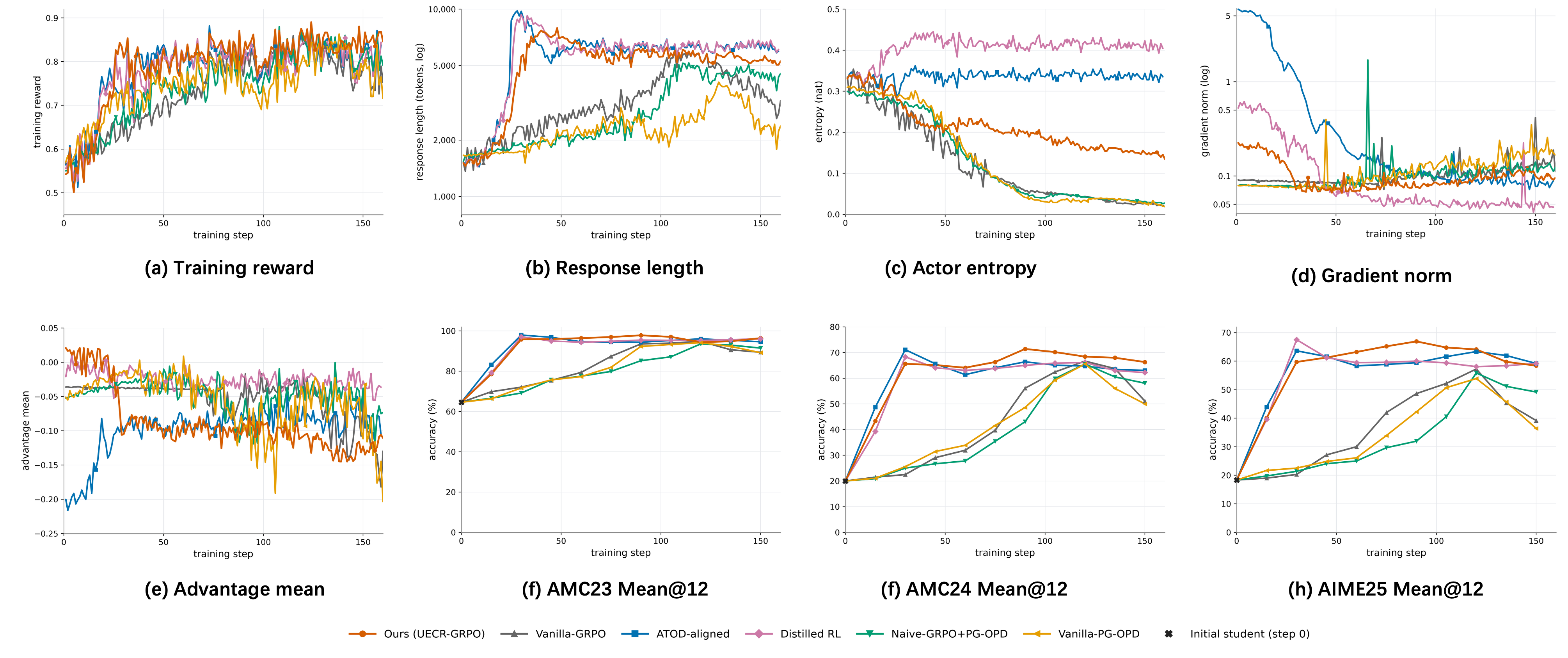}
\caption{Training dynamics for the aligned runs with the Qwen3-4B student
and Qwen3-4B-GRPO teacher, including
training reward, response length, actor entropy, gradient norm, advantage
statistics, and fixed-decoding validation accuracy.}
\label{fig:qwen4b-training}
\end{figure*}

\subsection{Remaining sensitivity and additional controls}

Appendix~\ref{app:alpha-sweep} reports the completed fixed-rollout sensitivity
audit for \(\alpha\). Retraining-based sensitivity to \(\alpha\), \(\rho\),
\(\tau_\delta\), and \(\tau_H\) remains future work; the offline audit does
not estimate the performance of those counterfactual training runs.

The following controls complement Table~\ref{tab:component-ablation}:
\begin{enumerate}[leftmargin=*,topsep=2pt,itemsep=1pt]
  \item set \(\rho=0\) to recover PUU exactly;
  \item set \(c_{i,t}=1\) while retaining the projection;
  \item set \(\mu_i^c=0\) to remove the additive constraint;
  \item apply \(w_{i,t}\) to the full \(A_i^U\), testing whether the teacher
  sequence component should remain unmodulated;
  \item set \(\alpha=0\) while retaining ECR;
  \item remove the signed gap to test whether confidence alone supplies a
  useful direction.
\end{enumerate}

\subsection{Efficiency and batch-fact auditing}

We report wall-clock, rollout time, scorer time, tokens/s, GPU0--7 peak memory,
and scorer overhead relative to Vanilla GRPO.  Logs contain unique prompts,
generated/rewarded/loss trajectories, valid response tokens, global steps, and
optimizer steps.  Seven-way data-parallel padding is removed before
response-wise normalization and restored only as masked rows.  A real empty
response is an error.  Any run that changes the intended global batch,
duplicates prompts, drops trajectories, or miscounts microbatches as global
steps fails the audit.

\subsection{Regression and smoke tests}

Before full training, the implementation verifies that \(\rho=0\) reproduces
the PUU loss and gradient, projection invariants hold on all valid masks,
raw-\(T=1\) teacher and old-policy scores align token by token, and reference
KL enters once.  Disabling teacher scoring must recover native GRPO without
loading the teacher.  CPU unit tests precede 1-, 5-, and 10-step GPU runs.

\section{Reproducibility Checklist}

\paragraph{Artifacts.}
Archive checkpoint, tokenizer, chat template, manifest, source commit, and
resolved-configuration hashes.  Record reward and group-normalization
conventions, decoding parameters, answer extraction, checkpoint selection,
and per-problem outputs.

\paragraph{Shared telemetry.}
For every method, log reward and advantage statistics, zero-variance groups,
PPO ratio/clip/KL statistics, response length and truncation, prompt,
trajectory, and token counts, and peak memory.  Teacher metrics are disabled, not zero, for Vanilla GRPO.

\paragraph{Method-specific telemetry.}
Additionally log \(\delta\), \(H^T\), \(c\), \(d\), \(\mu_i^c\),
\(q\), \(w\), teacher reward, decomposed PUU advantages, response-wise
budget error, and
\(\max_i|A_i^U-(A_i^{\task\mid U}+\alpha A_i^{T\mid U})|\).  State
numerical tolerances before training and treat violations as run failures.